\documentclass{article}

\usepackage[preprint]{icml2026}

\usepackage[utf8]{inputenc}
\usepackage[T1]{fontenc}
\usepackage{microtype}
\usepackage{graphicx}
\usepackage{booktabs}
\usepackage{multirow}
\usepackage{amsmath,amssymb,amsthm}
\usepackage{mathtools}
\usepackage{algorithm}
\usepackage{algorithmic}
\usepackage{xcolor}
\usepackage{hyperref}
\usepackage{enumitem}
\usepackage{natbib}

\hypersetup{
  colorlinks=true,
  linkcolor=blue,
  citecolor=blue,
  urlcolor=blue
}

\newcommand{\cZ}{\mathcal{Z}}
\newcommand{\cX}{\mathcal{X}}
\newcommand{\cY}{\mathcal{Y}}
\newcommand{\cM}{\mathcal{M}}
\newcommand{\cA}{\mathcal{A}}
\newcommand{\cW}{\mathcal{W}}
\newcommand{\R}{\mathbb{R}}
\newcommand{\E}{\mathbb{E}}
\newcommand{\Prb}{\mathbb{P}}
\newcommand{\norm}[1]{\left\lVert #1 \right\rVert}

\newcommand{\clip}{\mathrm{clip}}

\theoremstyle{plain}
\newtheorem{theorem}{Theorem}
\newtheorem{lemma}{Lemma}
\newtheorem{proposition}{Proposition}
\newtheorem{corollary}{Corollary}

\theoremstyle{definition}
\newtheorem{definition}{Definition}
\newtheorem{remark}{Remark}

\AtBeginDocument{%
  \setlength{\abovedisplayskip}{10pt plus 2pt minus 2pt}      
  \setlength{\belowdisplayskip}{10pt plus 2pt minus 2pt}      
  \setlength{\abovedisplayshortskip}{3pt plus 1pt minus 1pt} 
  \setlength{\belowdisplayshortskip}{4pt plus 2pt minus 1pt} 
}

\makeatletter
\renewcommand{\section}{\@startsection{section}{1}{\z@}{-0.10in}{0.01in}%
   {\normalfont\large\bf\raggedright}}
\renewcommand{\subsection}{\@startsection{subsection}{2}{\z@}{-0.08in}{0.005in}%
   {\normalfont\normalsize\bf\raggedright}}
\renewcommand{\paragraph}{\@startsection{paragraph}{4}{\z@}{0.8ex plus 0.2ex minus 0.2ex}{-1em}%
   {\normalfont\normalsize\bf}}
\makeatother

\icmltitlerunning{Risk-Equalized DP Synthetic Data}

\begin{document}

\twocolumn[
\icmltitle{Risk-Equalized Differentially Private Synthetic Data: Protecting Outliers by Controlling Record-Level Influence}

\begin{icmlauthorlist}
\icmlauthor{Amir Asiaee}{vumc-biostat}
\icmlauthor{Chao Yan}{vumc-bmi}
\icmlauthor{Zachary B.\ Abrams}{wustl}
\icmlauthor{Bradley A.\ Malin}{vumc-bmi}
\end{icmlauthorlist}

\icmlaffiliation{vumc-biostat}{Department of Biostatistics, Vanderbilt University Medical Center, 2525 West End Avenue, Nashville, TN 37203, USA}
\icmlaffiliation{vumc-bmi}{Department of Biomedical Informatics, Vanderbilt University Medical Center, Nashville, TN, USA}
\icmlaffiliation{wustl}{Institute for Informatics, Washington University, 4444 Forest Park Avenue, St.\ Louis, MO 63108, USA}

\icmlcorrespondingauthor{Amir Asiaee}{amir.asiaeetaheri@vumc.org}

\icmlkeywords{Differential Privacy, Synthetic Data, Outliers, Per-Instance Differential Privacy}

\vskip 0.3in
]
\printAffiliationsAndNotice{}

\begin{abstract}
When synthetic data is released, some individuals are harder to protect than others. A patient with a rare disease combination or a transaction with unusual characteristics stands out from the crowd. Differential privacy provides worst-case guarantees, but empirical attacks---particularly membership inference---succeed far more often against such outliers, especially under moderate privacy budgets and with auxiliary information.

This paper introduces \emph{risk-equalized DP synthesis}, a framework that prioritizes protection for high-risk records by reducing their influence on the learned generator. The mechanism operates in two stages: first, a small privacy budget estimates each record's ``outlierness''; second, a DP learning procedure weights each record inversely to its risk score. Under Gaussian mechanisms, a record's privacy loss is proportional to its influence on the output---so deliberately shrinking outliers' contributions yields tighter per-instance privacy bounds for precisely those records that need them most.

We prove end-to-end DP guarantees via composition and derive closed-form per-record bounds for the synthesis stage (the scoring stage adds a uniform per-record term). Experiments on simulated data with controlled outlier injection show that risk-weighting substantially reduces membership inference success against high-outlierness records; ablations confirm that targeting---not random downweighting---drives the improvement. On real-world benchmarks (Breast Cancer, Adult, German Credit), gains are dataset-dependent, highlighting the interplay between scorer quality and synthesis pipeline.
\end{abstract}

\section{Introduction}

Suppose a hospital releases a synthetic version of its patient database. Most patients---those with common diagnoses, typical demographics, and standard treatment paths---blend into the crowd. But consider a patient with a rare genetic condition who also happens to be unusually young and from a small geographic region. Even in synthetic data, this combination may appear only once. An adversary with auxiliary knowledge about such a patient can search the synthetic dataset, find a unique or near-unique match, and infer sensitive attributes that were never meant to be disclosed.

This scenario reflects real patterns in privacy research. Empirical studies show that privacy attacks against machine learning models and synthetic data succeed far more often against unusual records \citep{shokri2017membership,kulynych2022disparate}. For synthetic data specifically, outliers face substantially higher re-identification risk through linkage attacks \citep{trindade2024syntheticoutliers}, and membership inference attacks achieve higher accuracy against rare individuals even when differential privacy is applied \citep{vanbreugel2023domias,golob2025mamamia}---particularly under moderate privacy budgets.

Differential privacy offers a formal guarantee: changing any single record changes the output distribution by at most a bounded amount. But this is a \emph{worst-case} bound. The parameter $\varepsilon$ describes the maximum privacy loss any record could suffer---it says nothing about how that loss is distributed. In practice, typical records may experience privacy loss far below the nominal $\varepsilon$, while outliers approach or reach the worst case \citep{wang2019pidp,thudi2024gradients}. Standard DP mechanisms treat all records symmetrically, yet their empirical vulnerabilities differ.

\paragraph{The core idea.}
If outliers face higher empirical risk, a principled mechanism should compensate by giving them stronger protection. The question is how. Excluding outliers entirely would degrade utility for minority subgroups---often the populations most important to represent. Adding more noise globally would harm everyone.

Our solution is to control each record's \emph{influence} on the learned generator. Under Gaussian mechanisms, privacy loss is proportional to how much a record affects the output. By deliberately reducing the contribution of high-risk records---through careful weighting rather than exclusion---we can tighten their per-instance privacy bounds while still including their information. The result is a form of privacy equalization: records that are harder to protect receive less influence and hence lower privacy loss, while typical records contribute normally.

\paragraph{Approach.}
The mechanism operates in two stages. First, using a small privacy budget, we estimate each record's outlierness---how rare or atypical it is---via a differentially private scoring procedure (e.g., histogram frequencies or density estimates). Second, we train a DP synthesizer where each record's contribution is scaled inversely to its risk score. High-risk outliers contribute less to the learned model, which directly reduces their per-instance privacy loss. The two stages compose via standard DP accounting.

\paragraph{Contributions.}
\begin{enumerate}[leftmargin=*,itemsep=2pt]
\item A problem formulation that explicitly targets outlier protection: minimize utility loss subject to global DP and an upper bound on per-instance privacy loss for high-risk records.

\item REPS (Risk-Equalized Private Synthesis), a modular mechanism combining DP outlier scoring with risk-weighted DP learning, instantiated for Gaussian sufficient-statistics release.

\item Per-instance privacy theory: closed-form bounds showing how a record's privacy loss depends on its weight, plus a constructive schedule for choosing weights to cap per-record $\varepsilon_i$ for designated high-risk records.

\item Experiments on simulated data (with controlled outlier injection) and real-world tabular benchmarks (Breast Cancer Wisconsin, Adult, German Credit), evaluating both utility and privacy with explicit focus on membership inference stratified by outlierness.
\end{enumerate}

\paragraph{Scope.}
The goal is not to improve aggregate fidelity metrics---existing DP synthesizers already perform well on averages. Rather, we target the \emph{distribution} of privacy protection, ensuring that formal guarantees are realized more uniformly across the population. This matters especially in healthcare and social domains where outliers often represent vulnerable individuals who most need protection.

\section{Related Work}

This section situates our contribution relative to four research threads: empirical privacy vulnerability, per-instance and heterogeneous DP, DP synthetic data generation, and targeted high-risk approaches.

\subsection{Empirical privacy vulnerability}
Privacy attacks do not affect all individuals equally. Membership inference attacks \citep{shokri2017membership} reveal training set membership, but success rates vary dramatically across records. \citet{kulynych2022disparate} provide theoretical and empirical evidence of \emph{disparate vulnerability}: outliers and minorities face systematically higher attack success. \citet{jayaraman2021revisiting} caution that naive evaluations can mislead, but the core finding persists.

For synthetic data specifically, \citet{giomi2023anonymeter} develop Anonymeter, measuring singling-out, linkability, and inference risks. \citet{trindade2024syntheticoutliers} demonstrate that rare records are substantially easier to re-identify via linkage.

Two recent attacks sharpen the threat model. DOMIAS \citep{vanbreugel2023domias} targets \emph{local overfitting}: by comparing a target's density under the synthetic distribution versus an auxiliary reference, it achieves high membership inference accuracy---especially against uncommon samples. MAMA-MIA \citep{golob2025mamamia} exploits the structure of marginal-based generators (PrivBayes, MST, Private-GSD): knowing which algorithm was used allows substantially more efficient inference than prior attacks. These results motivate our experimental protocol: we evaluate membership inference stratified by outlierness and include attacks tailored to the generator class.

\citet{pollock2025freerecord} show that high-risk records can be identified efficiently from training artifacts (loss trajectories, gradient norms), suggesting practical outlier scoring within the DP pipeline.

\subsection{Per-instance and heterogeneous differential privacy}
Standard DP provides a worst-case guarantee: the same $\varepsilon$ applies to all records. Per-instance DP (pDP) \citep{wang2019pidp} refines this by deriving record-specific bounds $\varepsilon_i(D)$ that can be much tighter for typical records. \citet{thudi2024gradients} analyze DP-SGD's per-instance behavior, showing that records with similar neighbors enjoy tighter bounds than isolated outliers. These works are \emph{analytic}---they quantify heterogeneity post hoc rather than designing mechanisms to address it.

A separate line studies \emph{personalized} or \emph{individualized} privacy \citep{jorgensen2015pdp,boenisch2023idpsgd}, where users specify their desired protection level. Recent work on \emph{per-record DP} \citep{seeman2023perrecord,finley2024slowlyscaling} allows the privacy parameter to vary across records based on record characteristics (e.g., magnitude of contributions in skewed data). The key difference from our work: these frameworks let privacy vary to \emph{improve utility} on typical records while maintaining a budget; we use risk-based weighting to \emph{tighten protection} for outliers who face higher empirical attack success. The mechanisms are complementary: per-record DP provides the accounting framework, while our risk-adaptive weighting provides the policy for assigning tighter bounds to vulnerable records.

\subsection{DP synthetic data generation}
DP synthetic data methods fall into two families. \emph{Marginal-based approaches} like PrivBayes \citep{zhang2017privbayes} and PrivSyn \citep{zhang2021privsyn} measure noisy marginals and fit a distribution. \emph{Deep generative approaches} like PATE-GAN \citep{jordon2019pategan} and DP-CTGAN \citep{fang2022dpctgan} train neural generators with DP-SGD. Both families treat all records symmetrically: gradients are clipped uniformly, contributions are weighted equally, and the resulting privacy guarantee is worst-case.

Smooth sensitivity \citep{nissim2007smooth} calibrates noise to data-dependent rather than worst-case sensitivity, reducing noise when the data permits. Our approach is philosophically similar but inverted: rather than reducing noise globally, we \emph{increase protection locally} for high-risk records by shrinking their influence.

\subsection{Targeted high-risk approaches}
The closest prior work is $\varepsilon$-PrivateSMOTE \citep{carvalho2024privatesmote}, which addresses high-risk records by \emph{replacing} them with synthetic interpolations (SMOTE-style) under DP noise. This achieves strong linkability protection for targeted records with low computational cost. Our approach differs in three ways: (1) we reduce influence rather than replace---outliers still contribute to the learned model, preserving their distributional signal; (2) we provide formal per-instance privacy bounds, not just empirical risk reduction; (3) our framework applies to general DP learning (sufficient statistics, DP-SGD), not only interpolation-based synthesis.

\paragraph{Summary.}
Prior work documents outlier vulnerability empirically \citep{trindade2024syntheticoutliers,vanbreugel2023domias} or analyzes per-instance privacy theoretically \citep{wang2019pidp,thudi2024gradients}, but does not design mechanisms that explicitly control record-level influence based on risk. $\varepsilon$-PrivateSMOTE \citep{carvalho2024privatesmote} targets high-risk cases but replaces rather than weights them, and does not provide per-instance DP bounds. Our contribution is a DP synthesis mechanism with formal per-instance bounds and a constructive schedule for risk-based protection, bridging the empirical vulnerability literature with mechanism design.

\section{Preliminaries}

This section establishes notation and recalls the definitions we build on: differential privacy, per-instance privacy, and DP synthetic data generation.

\subsection{Data model and notation}
Let $D = (z_1,\dots,z_n) \in \cZ^n$ denote a dataset of $n$ records; we assume $n$ is public.
For tabular predictive modeling, each record is $z=(x,y)$ with features $x\in\cX$ and label $y\in\cY$.
To avoid ambiguity about normalization by $n$ under remove-one comparisons, we use the standard padding convention: let $\bot$ denote a fixed ``null'' record that contributes zero to all released statistics, and define $D^{(-i)} := (z_1,\dots,z_{i-1},\bot,z_{i+1},\dots,z_n)$. We treat $D$ and $D^{(-i)}$ as neighbors (equivalently, add/remove adjacency on variable-size datasets with a known size bound $n$ \citep{dworkroth2014book}).
A randomized mechanism $\cM$ maps datasets to outputs in some measurable space.

\subsection{Differential privacy}
\begin{definition}[Differential Privacy \citep{dwork2006calibrating}]
A mechanism $\cM$ satisfies $(\varepsilon,\delta)$-DP if for all neighboring $D \sim D'$ and all measurable $S$,
\[
\Prb[\cM(D)\in S] \le e^{\varepsilon}\Prb[\cM(D')\in S] + \delta.
\]
\end{definition}

\subsection{Per-instance privacy}
Standard DP is a worst-case guarantee: the same $\varepsilon$ bounds every record's privacy loss. Per-instance DP refines this by providing record-specific bounds.

\begin{definition}[Per-instance DP \citep{wang2019pidp}]
Fix dataset $D$ and index $i\in[n]$. Let $D^{(-i)}$ denote $D$ with record $z_i$ replaced by the null record $\bot$.
The mechanism $\cM$ satisfies $(\varepsilon_i(D),\delta)$-pDP for record $i$ on $D$ if for all events $S$,
\[
\Prb[\cM(D)\in S] \le e^{\varepsilon_i(D)}\Prb[\cM(D^{(-i)})\in S] + \delta.
\]
\end{definition}

If $\cM$ is $(\varepsilon,\delta)$-DP, then $\varepsilon_i(D)\le \varepsilon$ for all $D,i$. The per-instance bound can be much tighter for typical records and approach $\varepsilon$ only for outliers---this heterogeneity is what we exploit.

\subsection{DP synthetic data generation}
The standard paradigm learns generator parameters $\theta$ from $D$ using a DP algorithm $\cA$, then samples synthetic data:
\(
\theta \leftarrow \cA(D), \tilde D \sim p_\theta^{\otimes m}.
\)
Since any function of a DP output is DP (post-processing), releasing $\tilde D$ inherits the privacy guarantee of $\theta$.

\section{Problem Formulation}

This section formalizes what ``risk-equalized'' means. The primary goal is to \emph{cap per-instance privacy loss for high-risk records}: outliers should not approach the worst-case $\varepsilon$ just because they are rare. A secondary goal is to reduce inequality in per-instance bounds across all records.

\subsection{Outlier scores}
We abstract ``which records are high-risk'' via an outlierness functional.

\begin{definition}[Outlier score]
An \emph{outlier score} is a mapping $\mathsf{Out}: \cZ^n \times \cZ \to \R_{\ge 0}$
assigning each record $z_i$ a score $s_i := \mathsf{Out}(D, z_i)$.
Larger $s_i$ indicates higher re-identification or linkage risk.
\end{definition}

Concrete examples: negative log-density under a data model, $k$-nearest-neighbor distance, leverage scores, or small-cluster membership. In private settings, $s_i$ must be computed under DP or from public auxiliary data.

\subsection{Design objective}
Let $\cM$ denote a DP mechanism producing synthetic data $\tilde D$. We measure utility via downstream performance or fidelity metrics, and privacy via per-instance loss $\varepsilon_i(D)$.

\paragraph{Primary constraint: outlier protection.}
Let $\mathcal{H}_\tau := \{i: s_i \ge \tau\}$ be the high-risk set. The key constraint is:
\( \max_{i\in \mathcal{H}_\tau} \varepsilon_i(D) \;\le\; \varepsilon_{\mathrm{out}}, \)
where $\varepsilon_{\mathrm{out}} < \varepsilon$ is a tighter cap for outliers. This ensures that records most vulnerable to attack receive stronger formal protection.

\paragraph{Secondary constraint: inequality reduction.}
We may also bound the spread of per-instance losses:
\(
\mathsf{Ineq}(\varepsilon_{1:n}(D)) \le B,
\)
where $\mathsf{Ineq}$ could be the max-to-median ratio, Gini coefficient, or a quantile gap.

\paragraph{Optimization view.}
Combining these:
\begin{equation}
\label{eq:main-optim}
\begin{aligned}
\min_{\cM} \quad & \E[\mathsf{Util}(\cM(D))] \\
\text{s.t.}\quad & \cM \text{ is } (\varepsilon,\delta)\text{-DP},\\
& \max_{i\in \mathcal{H}_\tau} \varepsilon_i(D) \le \varepsilon_{\mathrm{out}}.
\end{aligned}
\end{equation}
Standard DP synthesis solves only the global constraint; our goal is to \emph{explicitly control} the distribution of $\varepsilon_i(D)$ across records via mechanism design.

\section{Method: Risk-Equalized DP Synthesis}

This section describes REPS (Algorithm~\ref{alg:redpsynth}), a modular mechanism achieving the objective in~\eqref{eq:main-optim}. The mechanism has two stages: (1) estimate each record's outlierness using a small privacy budget, then (2) train a DP generator where each record's contribution is weighted inversely to its risk. The stages compose additively.

\begin{algorithm}[t]
\caption{REPS}
\label{alg:redpsynth}
\begin{algorithmic}[1]
\REQUIRE Dataset $D$, budgets $(\varepsilon_s,\delta_s)$ and $(\varepsilon_t,\delta_t)$, weight map $g$, clip $C$.
\ENSURE Synthetic dataset $\tilde D$.
\STATE $\widehat{s}_{1:n} \leftarrow \cW(D)$ \hfill\textit{// DP risk scoring, budget $(\varepsilon_s,\delta_s)$}
\STATE $w_i \leftarrow g(\widehat{s}_i)$ for each $i$ \hfill\textit{// map scores to weights}
\STATE $\theta \leftarrow \cA(D; w_{1:n})$ \hfill\textit{// weighted DP learning, budget $(\varepsilon_t,\delta_t)$}
\STATE $\tilde D \sim p_\theta^{\otimes m}$
\STATE \textbf{return} $\tilde D$
\end{algorithmic}
\end{algorithm}

\subsection{Stage 1: DP outlier scoring}
Any DP mechanism that produces per-record risk scores suffices. We describe one concrete instantiation suitable for mixed tabular data: DP histogram frequencies.

Discretize continuous features into bins. For each column $j\in[d]$, release noisy bin counts:
\(
\tilde c_{j,b} = c_{j,b}(D) + \eta_{j,b}, \eta_{j,b}\sim \mathcal{N}(0,\sigma_s^2).
\)
Form estimated marginal probabilities $\tilde p_{j,b} = \max(\tilde c_{j,b},0)/\sum_{b'}\max(\tilde c_{j,b'},0)$ and define a rarity score:
\(
\widehat{s}_i = -\sum_{j=1}^d \log\big(\max(\tilde p_{j, \mathrm{bin}(x_{ij})}, p_{\min})\big).
\)
Records with low marginal probabilities receive higher scores. This is not a full density model but is stable, fast, and can be strengthened via pairwise marginals at additional privacy cost. Alternative scorers (DP clustering, DP $k$NN distances) are possible; our theory treats $\cW$ abstractly.

\paragraph{Sensitivity and accounting.}
Each record contributes to exactly one bin per feature, so the histogram vector for column $j$ has $\ell_2$ sensitivity 1 (add/remove changes one count by 1). For $d$ features released independently, we allocate $\varepsilon_s/d$ per feature under basic composition, or use tighter R\'enyi composition. In our experiments, we release all $d$ histogram vectors jointly as a single $\sum_j B_j$-dimensional vector (where $B_j$ is the number of bins for feature $j$); since each record affects exactly $d$ coordinates (one per feature), the $\ell_2$ sensitivity is $\sqrt{d}$, and we calibrate $\sigma_s = \sqrt{d} \cdot \sqrt{2\log(1.25/\delta_s)}/\varepsilon_s$.

\subsection{Stage 1.5: Mapping scores to weights}
Given scores $\widehat{s}_i$, define weights $w_i\in(0,1]$ that decrease with risk. Two useful families:
\[
\begin{aligned}
\textbf{(cap)}\quad w_i &= \min\!\left(1, \frac{\tau}{\widehat{s}_i+\tau}\right),\\
\textbf{(hinge-exp)}\quad w_i &= \exp\big(-\gamma(\widehat{s}_i - t)_+\big).
\end{aligned}
\]
The cap mapping yields gradual protection; hinge-exp yields sharper protection for extreme outliers.

For equalization, set $w_i$ so that $w_i \cdot \alpha_i \le \tau_{\mathrm{inf}}$ for all $i$, where $\alpha_i$ is record $i$'s influence proxy (e.g., clipped contribution norm). This directly caps per-instance privacy loss (Theorem~\ref{thm:pidp-gaussian}).

\subsection{Stage 2: Risk-weighted DP learning}

We instantiate two variants: one-shot Gaussian release (for marginal-based synthesis) and iterative DP-SGD (for deep generators). Both rely on \emph{clipping}---bounding each record's contribution to norm at most $C$---which is standard in DP learning to control sensitivity.

\paragraph{Instantiation A: weighted clipped statistics.}
Let $\phi:\cZ\to\R^k$ encode per-record sufficient statistics (e.g., feature values, squared values for variances, one-hot category indicators). Define the weighted clipped aggregate:
\begin{equation}
\label{eq:weighted-aggregate}
\begin{aligned}
f_w(D) &:= \frac{1}{n}\sum_{i=1}^n w_i \cdot \clip(\phi(z_i), C),\\
\clip(u,C) &:= u\cdot \min\{1, C/\norm{u}_2\}.
\end{aligned}
\end{equation}
Release $\tilde f = f_w(D) + \mathcal{N}(0,\sigma^2 I_k)$, fit a graphical model or exponential family to $\tilde f$, and sample synthetic data.

\paragraph{Instantiation B: risk-weighted DP-SGD.}
For a generator loss $\ell(\theta; z)$, standard DP-SGD clips per-example gradients and adds noise. We modify it to incorporate risk weights:
\[
\begin{aligned}
\bar g(\theta) &= \frac{1}{|B|}\sum_{i\in B} w_i\cdot \clip(\nabla_\theta \ell(\theta; z_i), C),\\
\tilde g(\theta) &= \bar g(\theta) + \mathcal{N}(0, \sigma^2 C^2 I).
\end{aligned}
\]
Update $\theta \leftarrow \theta - \eta \tilde g(\theta)$. This reduces outliers' influence on the learned generator while still including their signal.

\section{Theory}

This section provides the formal underpinnings for REPS. The results fall into two categories: (1) \emph{standard} DP machinery---composition and Gaussian mechanism bounds---applied to our setting, and (2) \emph{new} contributions---the constructive protection schedule (Proposition~\ref{prop:targeted}) and the interpretation of weighting as a per-instance privacy lever. Our per-instance bounds are tightest for Instantiation A (one-shot Gaussian release); for DP-SGD we provide intuition and validate empirically.
\textbf{All proofs are deferred to Appendix~\ref{app:proofs}.}

\subsection{End-to-end DP via composition}
\begin{theorem}[Composition]
\label{thm:composition}
If Stage 1 scorer $\cW$ is $(\varepsilon_s,\delta_s)$-DP and Stage 2 learner $\cA(\cdot;w)$ is $(\varepsilon_t,\delta_t)$-DP for any fixed $w$, then Algorithm~\ref{alg:redpsynth} is $(\varepsilon_s+\varepsilon_t,\delta_s+\delta_t)$-DP.
\end{theorem}

\begin{corollary}[End-to-end per-instance bound]
\label{cor:end-to-end}
Under Theorem~\ref{thm:composition}, Algorithm~\ref{alg:redpsynth} satisfies $(\varepsilon_{i,\mathrm{tot}}(D),\delta_s+\delta_t)$-pDP for each record $i$, with
\[
\varepsilon_{i,\mathrm{tot}}(D) \le \varepsilon_s + \varepsilon_{i,t}(D),
\]
where $\varepsilon_{i,t}(D)$ is any valid per-instance bound for the Stage 2 release (e.g., Corollary~\ref{thm:pidp-gaussian} with $\delta_t$).
\end{corollary}

\subsection{Influence and per-instance privacy: general framework}

The key insight is that under Gaussian mechanisms, a record's privacy loss is controlled by its \emph{influence}---how much removing it changes the output. This is a direct application of standard Gaussian mechanism analysis \citep{dworkroth2014book} to the per-instance setting.

\begin{definition}[Record-level influence]
For a function $f:\cZ^n \to \R^k$ and dataset $D$, define record $i$'s influence as
\[
\alpha_i(f,D) := \norm{f(D) - f(D^{(-i)})}_2,
\]
where $D^{(-i)}$ replaces $z_i$ with $\bot$.
\end{definition}

\begin{theorem}[Influence-to-privacy bound]
\label{thm:influence-privacy}
Let $\cM(D) = f(D) + \mathcal{N}(0,\sigma^2 I_k)$. For any $\delta \in (0,1)$, record $i$ on dataset $D$ satisfies $(\varepsilon_i(D),\delta)$-pDP with
\[
\varepsilon_i(D) \;\le\; \frac{\alpha_i(f,D)}{\sigma}\sqrt{2\log(1.25/\delta)}.
\]
\end{theorem}

\begin{remark}[What is new]
The bound itself is standard. Our contribution is the \emph{design interpretation}: by constructing mechanisms where $\alpha_i$ is small for high-risk records, we obtain tighter per-instance bounds for precisely those records. This motivates the risk-weighted aggregates below.
\end{remark}

\subsection{Weighted aggregates: instantiation}

For the weighted clipped aggregate $f_w(D) = \frac{1}{n}\sum_{i=1}^n w_i \cdot \clip(\phi(z_i), C)$, we take $D^{(-i)}$ to replace $z_i$ with the null record $\bot$ and assume $\phi(\bot)=0$, so $\clip(\phi(\bot),C)=0$. This yields a clean influence bound under the same normalization by $n$.

\begin{lemma}[Influence under weighting]
\label{lem:record-sens}
\[
\alpha_i(f_w, D) = \norm{f_w(D) - f_w(D^{(-i)})}_2 \le \frac{w_i C}{n}.
\]
\end{lemma}

Combining with Theorem~\ref{thm:influence-privacy}:

\begin{corollary}[Per-instance bound for weighted release]
\label{thm:pidp-gaussian}
For $\cM(D) = f_w(D) + \mathcal{N}(0,\sigma^2 I_k)$ and any $\delta_t\in(0,1)$, record $i$ satisfies $(\varepsilon_i(D),\delta_t)$-pDP with
\[
\varepsilon_i(D) \le \frac{w_i C}{n\sigma}\sqrt{2\log(1.25/\delta_t)}.
\]
\end{corollary}

Reducing $w_i$ directly reduces record $i$'s per-instance privacy bound.

\subsection{Outlier protection schedule}

The following result shows how to protect a designated high-risk set while maintaining global DP calibration.

\begin{proposition}[Targeted protection]
\label{prop:targeted}
Let $\mathcal{H} \subset [n]$ be the high-risk set. Suppose we set
\[
w_i = \begin{cases}
\tau_{\mathrm{out}}/C & i \in \mathcal{H} \\
1 & i \notin \mathcal{H}
\end{cases}
\]
for some $\tau_{\mathrm{out}} \le C$. Then:
\begin{enumerate}[leftmargin=*,itemsep=1pt]
\item Records in $\mathcal{H}$ satisfy {\small}$\varepsilon_i(D) \le \varepsilon_{\mathrm{out}} := \frac{\tau_{\mathrm{out}}}{n\sigma}\sqrt{2\log(1.25/\delta_t)}$.
\item The global sensitivity is $\Delta_{\max} = C/n$ (unchanged from uniform weighting).
\item The mechanism remains $(\varepsilon,\delta_t)$-DP with $\varepsilon = \frac{C}{n\sigma}\sqrt{2\log(1.25/\delta_t)}$.
\end{enumerate}
\end{proposition}

The proposition shows that protecting outliers does \emph{not} require increasing global noise---it only requires identifying them and reducing their weights. The trade-off is in utility: outliers contribute less to the learned model.

\begin{remark}[Comparison to naive equalization]
A naive approach would cap \emph{all} weights at $\tau_{\mathrm{out}}/C$, yielding $\varepsilon_i \le \varepsilon_{\mathrm{out}}$ for all $i$. But this reduces sensitivity globally and either (a) allows less noise for the same global $\varepsilon$, or (b) requires the same noise for a tighter global $\varepsilon$. Proposition~\ref{prop:targeted} achieves outlier protection without this trade-off by targeting only high-risk records.
\end{remark}

\subsection{Utility: bias--variance trade-off}

Risk-weighting introduces bias from downweighting. Let $\bar\phi(D) = \frac{1}{n}\sum_i \phi(z_i)$ be the unweighted mean.

\begin{proposition}[Bias bound]
\label{prop:bias} 
\begin{align} \nonumber
    &\displaystyle\norm{f_w(D) - \bar\phi(D)}_2 \\ \nonumber
    \le 
    &\frac{1}{n}\sum_{i=1}^n (1-w_i)\norm{\clip(\phi(z_i),C)}_2 + \\ \nonumber
    &\frac{1}{n}\sum_{i=1}^n \norm{\phi(z_i) - \clip(\phi(z_i),C)}_2.
\end{align}
\end{proposition}

The first term is the cost of downweighting; the second is the cost of clipping. When only a small fraction of records (outliers) have $w_i < 1$, the bias is small. The variance is $k\sigma^2$, determined by global DP calibration.

\subsection{Extension to DP-SGD (informal)}

For DP-SGD, the influence $\alpha_i$ depends on gradient similarity across iterations \citep{thudi2024gradients}. Risk-weighting shrinks the clipped gradient contribution of high-risk records at each step, directly reducing their cumulative influence. A full per-instance accounting extends \citet{feldman2021renyifilter}; we leave formal statements to future work and validate empirically via attack success stratified by outlierness.

\section{Experiments}

We evaluate REPS along four axes: (Q1) whether, under a fixed global DP budget $(\varepsilon,\delta)$, it reduces outlier-targeted membership inference compared to standard DP synthesis; (Q2) whether it reduces risk inequality while preserving utility; (Q3) how the utility--privacy frontier varies with risk-weight aggressiveness $\gamma$ and budget allocation; and (Q4) whether gains are specific to targeting outliers (vs.\ random downweighting).

\subsection{Threat Model}
\label{sec:threat-model}

We consider an adversary who observes the released synthetic dataset $\tilde{D}$ and aims to infer information about individuals in the original training set $D$.

\paragraph{Adversary knowledge.}
The adversary has access to: (i)~the full synthetic dataset $\tilde{D}$; (ii)~knowledge of the synthesis algorithm class (but not random seeds or noise realizations); and (iii)~auxiliary information about some target individuals, such as partial quasi-identifiers or known membership in the training population.

\paragraph{Adversary goals.}
We focus on membership inference (MIA), aligned with inference risks in the Anonymeter framework \citep{giomi2023anonymeter}. MIA asks whether a target record $z^*$ is in $D$; we measure success via AUC and \emph{advantage} $= 2\cdot \text{AUC}-1$ \citep{yeom2018privacy}. Throughout, we stratify results by \emph{outlier decile} (defined by $k$NN distance in feature space) to quantify disproportionate risk for high-decile records.

\paragraph{Attack implementations.}
We implement two MIA variants: (1)~a simple distance-based baseline where records closer to synthetic data are predicted as members, and (2)~a DOMIAS-style attack \citep{vanbreugel2023domias} that computes the ratio of distances to synthetic data vs.\ a reference set of non-members.

\paragraph{Success metrics.}
We report: (i)~overall MIA advantage (absolute value); (ii)~top-decile MIA advantage (attack success on the 10\% most outlying records); and (iii)~risk inequality ratio = top-decile / median-decile advantage.

\subsection{Datasets}
\label{sec:datasets}

We evaluate on four datasets spanning different scales and outlier characteristics:

\paragraph{Simulated dataset (controlled outliers).}
We construct a synthetic tabular dataset ($n=6{,}000$) with explicit outlier injection to enable ground-truth evaluation. The dataset contains 6 continuous features drawn from a 2-component Gaussian mixture (inliers) and 3 categorical features. We inject 2\% outliers ($n_{\text{out}}=120$) with: (i)~rare categorical values (``Z'', ``Q'', ``R'') that appear only in outliers; and (ii)~shifted continuous features ($\mu = 2.4$ vs.\ $\mu \in \{-0.8, 0.8\}$ for inliers). A binary label is generated via logistic model with signal from both continuous and categorical features.

\paragraph{Breast Cancer Wisconsin.}
UCI Breast Cancer dataset ($n=569$) with 30 continuous features from digitized cell nucleus images. We bin 3 randomly selected features into 5 quantile-based categories to create a mixed-type dataset. Natural outliers arise from rare tumor characteristics.

\paragraph{Adult (Census Income).}
UCI Adult dataset ($n=45{,}222$) with 6 continuous and 8 categorical features predicting income $>$\$50K. This large-scale dataset tests scalability and contains natural demographic outliers (rare occupation/education combinations).

\paragraph{German Credit.}
UCI German Credit dataset ($n=1{,}000$) with 7 continuous and 13 categorical features predicting creditworthiness. The high categorical dimensionality creates many rare feature combinations that serve as natural outliers.

\paragraph{Data splits and evaluation.}
Each dataset is split 56\%/14\%/30\% into train/validation/test sets. The validation set is used for hyperparameter selection ($\gamma$) to avoid test leakage. Outlier scores are computed via $k$NN distance ($k=10$) in standardized feature space; records are assigned to deciles with decile 10 containing the 10\% most outlying. All experiments use 5 random seeds; we report mean results.

\subsection{Methods and Baselines}
\label{sec:methods}

We compare REPS against baselines and ablations:

\paragraph{Non-private baseline.}
We run the same Naive Bayes synthesis without any noise to establish utility upper bounds. This baseline also reveals the privacy leakage when no DP protection is applied.

\paragraph{DP uniform (no scoring).}
All records receive weight $w_i = 1$ and we spend the full budget $(\varepsilon,\delta)$ on the synthesis stage. This is the standard DP baseline under the same total privacy budget.

\paragraph{DP uniform (matched split).}
To isolate the effect of risk-weighting under matched noise, we also report a split-budget uniform baseline that uses only $(\varepsilon_t,\delta_t)=(0.9\varepsilon,\delta/2)$ for synthesis. The remaining $(\varepsilon_s,\delta_s)=(0.1\varepsilon,\delta/2)$ budget is unused.

\paragraph{REPS (ours).}
Two-stage mechanism as described in Algorithm~\ref{alg:redpsynth}: Stage 1 computes DP histogram-based scores (16 bins for continuous features; exact categories for categorical features) using $(\varepsilon_s,\delta_s)=(0.1\varepsilon,\delta/2)$; Stage 1.5 maps scores to weights via $w_i = \exp(-\gamma(s_i-t)_+)$ with threshold $t$ at the 90th percentile; Stage 2 releases risk-weighted clipped statistics (Eq.~\ref{eq:weighted-aggregate}) using $(\varepsilon_t,\delta_t)=(0.9\varepsilon,\delta/2)$. We select $\gamma \in \{0.5, 1.0, 2.0, 4.0\}$ using a held-out validation set: choose the largest $\gamma$ whose validation TSTR AUROC is within 0.02 of the DP uniform (no scoring) baseline. This tuning step is \emph{not} included in the DP accounting; in a deployed mechanism, $\gamma$ must be fixed a priori or selected using a DP-safe procedure.

\paragraph{Ablation: random downweighting.}
Apply the same weight distribution as REPS but shuffle weights randomly across records. This tests whether the benefit comes from reducing \emph{total} contribution variance or specifically from targeting outliers.

\paragraph{Ablation: hard removal.}
Remove all records with DP scores above the 90th percentile threshold, then run standard DP synthesis on remaining records. This compares soft downweighting against hard exclusion.

\paragraph{Additional baselines.}
Our goal is to isolate the effect of risk-weighting in a controlled sufficient-statistics pipeline. We therefore emphasize comparisons against the two uniform DP baselines above and targeted ablations. Extending REPS weighting to canonical DP synthesizers (e.g., PGM/MST or DP-SGD-based generators) is an important direction for future work.

\subsection{Privacy Budget and Accounting}
\label{sec:privacy-budget}

We evaluate across privacy budgets $\varepsilon \in \{0.5, 1.0, 2.0, 4.0\}$ with $\delta = 1/n^2$ (standard choice for $n$-record datasets). For REPS (and the matched-split baseline), we split privacy as $\varepsilon = \varepsilon_s + \varepsilon_t$ with $\varepsilon_s = 0.1\varepsilon$, $\varepsilon_t = 0.9\varepsilon$, and $\delta_s = \delta_t = \delta/2$. The DP uniform (no scoring) baseline uses the full $(\varepsilon,\delta)$ for synthesis.
End-to-end DP follows from Theorem~\ref{thm:composition}. For a Gaussian mechanism with $\ell_2$ sensitivity $\Delta_2$, we use noise $\sigma = \Delta_2 \cdot \sqrt{2\log(1.25/\delta)}/\varepsilon$ \citep{dworkroth2014book}.

\subsection{Utility Evaluation}
\label{sec:utility-eval}

We measure utility via \emph{train-on-synthetic, test-on-real} (TSTR): we generate $\tilde{D}$ with $|\tilde{D}| = |D_{\text{train}}|$, train logistic regression on $\tilde{D}$, and report AUROC on the held-out real test set. Full details are in Appendix~\ref{app:experiments}.

\subsection{Results: Membership Inference by Outlier Decile}
\label{sec:results-mia}

Full numerical results across all datasets and baselines are provided in Appendix~\ref{app:experiments}. We report two uniform DP baselines: DP uniform (no scoring), which spends the full $(\varepsilon,\delta)$ on synthesis, and DP uniform (matched split), which matches REPS's synthesis-stage budget.

\paragraph{Q1: Does REPS reduce outlier-targeted privacy risk?}

On the simulated dataset with ground-truth injected outliers, REPS reduces top-decile MIA advantage across all tested $\varepsilon$ values (Figure~\ref{fig:mia-eps}). For example, at $\varepsilon=4.0$ the top-decile advantage drops from 0.760 (DP uniform, no scoring) to 0.728 (REPS). On the real-world benchmarks, results are mixed: Breast Cancer shows modest improvements at higher $\varepsilon$, while Adult and Credit do not show consistent top-decile improvements under the same total budget (Appendix~\ref{app:experiments}; per-decile curves in Figure~\ref{fig:decile-curves}). This underscores that risk-weighting is not a universal ``free win'' against membership inference in all tabular settings.

\paragraph{Q2: Does it reduce risk inequality while preserving utility?}

Figure~\ref{fig:utility-eps} shows the utility--privacy trade-off on the simulated dataset. Risk-weighting can reduce top-decile vulnerability with modest utility loss, and can modestly reduce the top/median inequality ratio in settings where outliers are well captured by the scorer. However, on datasets where the vulnerability is driven by higher-order feature interactions, we observe limited or negative changes in the inequality ratio (Appendix~\ref{app:experiments}).

\begin{figure}[t]
\centering
\includegraphics[width=0.92\linewidth]{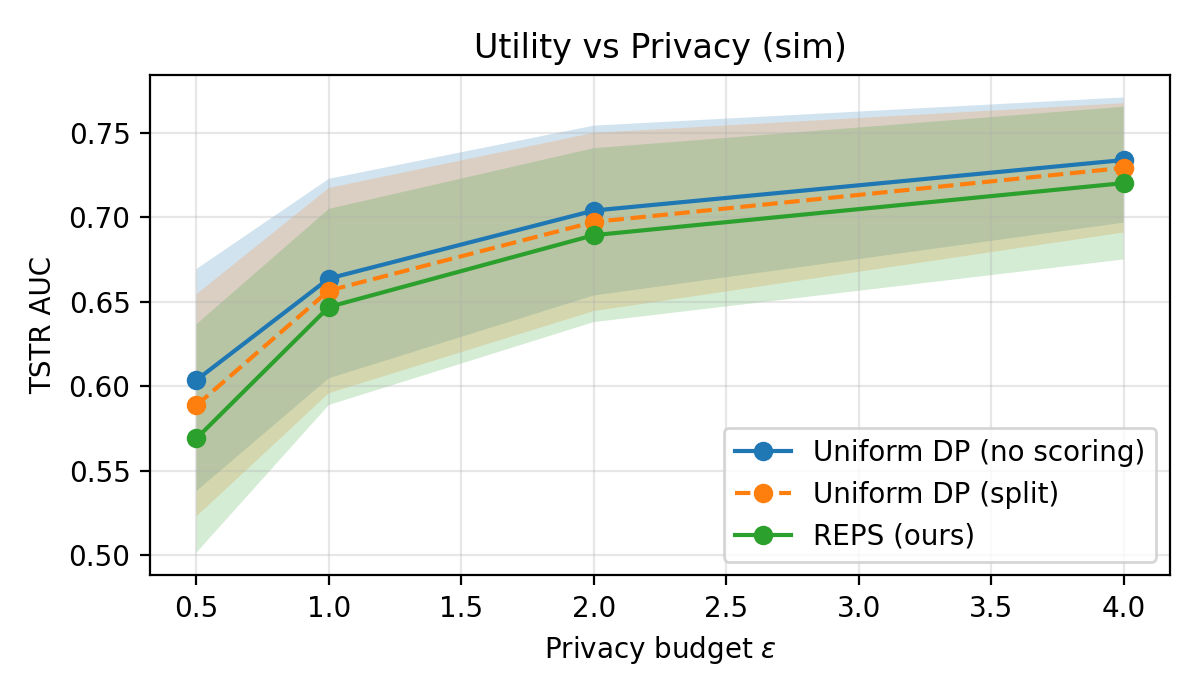}
\caption{Utility (TSTR AUROC) versus privacy budget on the simulated dataset. REPS shows slightly lower utility---the expected cost of downweighting outliers to protect their privacy.}
\label{fig:utility-eps}
\end{figure}

\begin{figure}[t]
\centering
\includegraphics[width=0.92\linewidth]{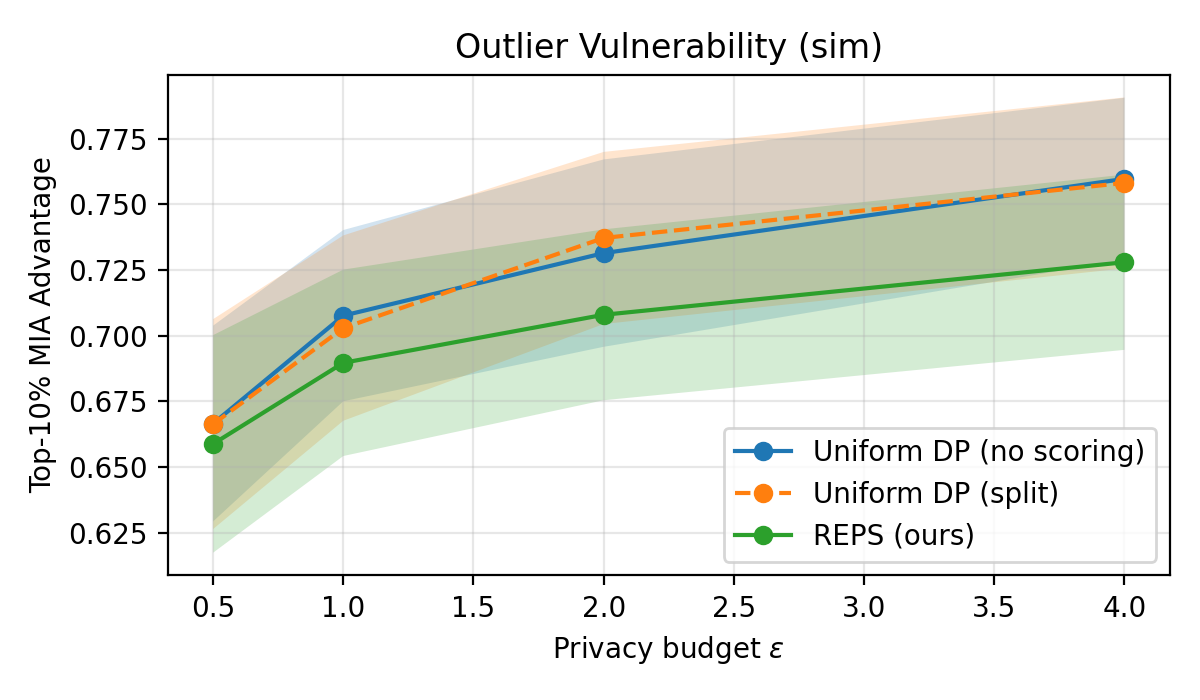}
\caption{Top-10\% MIA advantage versus privacy budget on the simulated dataset. Lower is better: REPS reduces attacker success against the most vulnerable records.}
\label{fig:mia-eps}
\end{figure}

\paragraph{Q3: Privacy-utility frontier.}
As $\varepsilon$ increases, all methods improve in utility (Figure~\ref{fig:utility-eps}). At low $\varepsilon$, utility constraints limit aggressive downweighting; at higher $\varepsilon$, more aggressive weighting is feasible.

\paragraph{Q4: Ablation---is targeting essential?}
Shuffling weights randomly across records (same distribution, random assignment) is less effective than targeted weighting on the simulated dataset, confirming that identifying high-risk records matters. Results on real-world benchmarks are mixed (Appendix~\ref{app:experiments}).

\section{Discussion}

\paragraph{Outliers vs.\ minorities.}
Outliers may represent marginalized populations whose accurate representation matters for fairness. Downweighting improves their privacy but may degrade fidelity for these groups. We recommend treating $\gamma$ as a policy knob balancing privacy equity against representation.

\paragraph{Scorer quality.}
Scorer quality (Figures~\ref{fig:scorer-corr}--\ref{fig:scorer-recall}) improves with $\varepsilon$ on simulated and Breast Cancer data. On Adult, scorer quality is high even at $\varepsilon=0.5$, yet MIA improvements are limited---suggesting effectiveness depends on whether the attack exploits the same statistics we control (e.g., higher-order interactions on Credit).

\paragraph{Scope and attack model.}
We focus on sufficient-statistics synthesis for clean per-instance analysis; DP-SGD bounds require integrating data-dependent analyses \citep{thudi2024gradients,feldman2021renyifilter}. Our distance-based MIA may underestimate sophisticated attackers \citep{carlini2022membership}; MAMA-MIA \citep{golob2025mamamia} suggests sufficient-statistics methods may be more vulnerable, which REPS mitigates by reducing outlier influence.

\paragraph{Broader applicability.}
The principle extends to DP-SGD, federated learning with outlier clients, and query answering---anywhere per-instance loss scales with influence.

\section{Conclusion}

We introduced risk-equalized DP synthetic data generation, a framework that protects high-risk outliers by reducing their influence on the learned synthesizer. Our two-stage mechanism---DP outlier scoring followed by risk-weighted DP learning---achieves formal $(\varepsilon,\delta)$-DP with improved per-instance privacy bounds for vulnerable records.

Experiments across four datasets with five random seeds show clear reductions in top-decile MIA advantage on the controlled outlier-injection benchmark, and smaller but sometimes positive gains on Breast Cancer at higher $\varepsilon$. On Adult and Credit, we observe mixed results under a fixed total budget, highlighting that influence control in a sufficient-statistics pipeline does not automatically translate to uniform empirical robustness against membership inference across all tabular domains. Ablations suggest that, when improvements occur, they come from correctly targeting high-risk records rather than random downweighting.

This work bridges the empirical observation that outliers face disproportionate privacy risk in synthetic data \citep{trindade2024syntheticoutliers,kulynych2022disparate} with mechanism design tools from per-instance differential privacy \citep{wang2019pidp}. The main practical takeaway is that influence-weighting provides a simple lever to shape per-instance guarantees; empirical outlier-robustness gains depend on scorer quality and on the synthesis/attack setting.

\section*{Broader Impact}
This paper advances differentially private synthetic data generation with the explicit goal of reducing disproportionate privacy risk for high-outlierness individuals. A primary positive impact is enabling safer data sharing in domains such as healthcare and finance where rare records are both valuable and vulnerable; risk-equalized mechanisms can help organizations deploy DP synthesis while providing more uniform protection across the population.
Potential negative impacts stem from the same mechanism: downweighting outliers may reduce fidelity for minority subpopulations and affect downstream fairness, particularly if synthetic data is used for policy or clinical decisions. We recommend treating the risk-weight aggressiveness as a policy knob and reporting subgroup utility alongside privacy metrics, and we caution against using synthetic data as the sole basis for high-stakes decisions without domain-specific auditing.

\bibliographystyle{plainnat}
\bibliography{ref2}

\appendix
\section{Proofs}
\label{app:proofs}

\begin{proof}[Proof of Theorem~\ref{thm:composition}]
Stage 1 releases $\widehat{s}_{1:n} \leftarrow \cW(D)$ and is $(\varepsilon_s,\delta_s)$-DP by assumption. The weights $w_{1:n} = g(\widehat{s}_{1:n})$ are deterministic post-processing of the Stage 1 output. Stage 2 takes $w_{1:n}$ as auxiliary input and releases $\theta \leftarrow \cA(D;w_{1:n})$; since $\cA(\cdot;w)$ is $(\varepsilon_t,\delta_t)$-DP for any fixed $w$, adaptive composition \citep{dworkroth2014book} yields that $(\widehat{s}_{1:n},\theta)$ is $(\varepsilon_s+\varepsilon_t,\delta_s+\delta_t)$-DP. Sampling $\tilde D \sim p_\theta^{\otimes m}$ is further post-processing.
\end{proof}

\begin{proof}[Proof of Corollary~\ref{cor:end-to-end}]
Stage 1 is $(\varepsilon_s,\delta_s)$-DP, hence also $(\varepsilon_s,\delta_s)$-pDP for every record. Conditional on any fixed weights $w$, Stage 2 admits a per-instance bound $(\varepsilon_{i,t}(D),\delta_t)$-pDP by assumption (e.g., Corollary~\ref{thm:pidp-gaussian} for a Gaussian release). Applying adaptive composition to the two-stage release yields $(\varepsilon_s+\varepsilon_{i,t}(D),\delta_s+\delta_t)$-pDP for record $i$.
\end{proof}

\begin{proof}[Proof of Theorem~\ref{thm:influence-privacy}]
Fix record $i$ and let $D' = D^{(-i)}$. The function $f$ has $\ell_2$ sensitivity $\Delta_2 = \norm{f(D)-f(D')}_2 = \alpha_i(f,D)$ for this particular neighboring pair. The Gaussian mechanism with noise $\mathcal{N}(0,\sigma^2 I_k)$ therefore satisfies $(\varepsilon_i(D),\delta)$-DP for this pair with $\varepsilon_i(D) \le \Delta_2 \sqrt{2\log(1.25/\delta)}/\sigma$ \citep{dworkroth2014book}.
\end{proof}

\begin{proof}[Proof of Lemma~\ref{lem:record-sens}]
By definition, $D^{(-i)}$ replaces $z_i$ with $\bot$ and we take $\clip(\phi(\bot),C)=0$, so
\[
f_w(D) - f_w(D^{(-i)}) = \frac{1}{n} w_i \cdot \clip(\phi(z_i),C).
\]
Since $\norm{\clip(\phi(z_i),C)}_2 \le C$ by definition of clipping, we have $\alpha_i(f_w,D) \le w_i C/n$.
\end{proof}

\begin{proof}[Proof of Corollary~\ref{thm:pidp-gaussian}]
Apply Theorem~\ref{thm:influence-privacy} to $f=f_w$ and use Lemma~\ref{lem:record-sens} to bound $\alpha_i(f_w,D) \le w_i C/n$.
\end{proof}

\begin{proof}[Proof of Proposition~\ref{prop:targeted}]
For $i \in \mathcal{H}$, Corollary~\ref{thm:pidp-gaussian} gives $\varepsilon_i(D) \le \frac{w_i C}{n\sigma}\sqrt{2\log(1.25/\delta_t)} = \varepsilon_{\mathrm{out}}$. For global sensitivity, the worst-case influence bound is $\max_i w_i C/n = C/n$ since $w_i=1$ for $i\notin\mathcal{H}$. The Gaussian mechanism calibrated to sensitivity $C/n$ is therefore $(\varepsilon,\delta_t)$-DP with $\varepsilon = \frac{C}{n\sigma}\sqrt{2\log(1.25/\delta_t)}$.
\end{proof}

\begin{proof}[Proof of Proposition~\ref{prop:bias}]
Write $\bar\phi(D)=\frac{1}{n}\sum_{i=1}^n \phi(z_i)$ and expand:
\[
\begin{aligned}
f_w(D) - \bar\phi(D)
&= \frac{1}{n}\sum_{i=1}^n \big(w_i\clip(\phi(z_i),C) - \phi(z_i)\big)\\
&= \frac{1}{n}\sum_{i=1}^n (w_i-1)\clip(\phi(z_i),C)\\
&\quad + \frac{1}{n}\sum_{i=1}^n \big(\clip(\phi(z_i),C)-\phi(z_i)\big).
\end{aligned}
\]
Taking norms and applying the triangle inequality, and using $w_i\in(0,1]$ so $\norm{(w_i-1)u}_2 = (1-w_i)\norm{u}_2$, yields the stated bound.
\end{proof}

\section{Extended Experimental Details}
\label{app:experiments}
This appendix provides reproducibility details (pipeline, DP accounting, attacks) and full numerical results (mean over 5 seeds).

\subsection{Implementation and preprocessing}
\paragraph{Feature processing.}
Continuous features are standardized using the training split (zero mean, unit variance) and clipped to $[-3,3]$ before computing DP sufficient statistics. Categorical features are treated as strings and one-hot encoded only for evaluation models/attacks.

\paragraph{Train/validation/test splits.}
We split each dataset into 56\%/14\%/30\% train/validation/test with stratification on the label. The validation split is used only for selecting $\gamma$ for REPS, and we report mean metrics over 5 random seeds.

\subsection{DP accounting and hyperparameters}
\paragraph{Budgets.}
We use $\delta = 1/n^2$ where $n$ is the training-set size. For REPS we split $(\varepsilon,\delta)$ across stages as $(\varepsilon_s,\delta_s)=(0.1\varepsilon,\delta/2)$ for scoring and $(\varepsilon_t,\delta_t)=(0.9\varepsilon,\delta/2)$ for synthesis. The DP uniform (no scoring) baseline uses the full $(\varepsilon,\delta)$ for synthesis.

\paragraph{Scoring.}
Stage 1 computes DP histogram rarity scores using 16 bins for continuous features (uniform bin edges over $[-3,3]$) and exact categories for categorical features, releasing all histograms jointly with Gaussian noise and $\ell_2$ sensitivity $\sqrt{d}$ for $d$ features. We report scorer quality by (i) Spearman correlation between DP scores and non-private $k$NN outlier scores, and (ii) Recall@Top-10\% (fraction of $k$NN top-decile records that are also in the DP-score top decile).

\paragraph{Weights and tuning.}
We map scores to weights via $w_i=\exp(-\gamma(s_i-t)_+)$ with threshold $t$ at the 90th percentile. We select $\gamma \in \{0.5,1.0,2.0,4.0\}$ by choosing the largest value whose validation TSTR AUROC is within 0.02 of the DP uniform (no scoring) baseline. This selection is \emph{not} included in the DP accounting; in a deployed mechanism, $\gamma$ must be fixed a priori or chosen using a DP selection procedure.

\subsection{Synthesis model (Instantiation A)}
\paragraph{Conditional Gaussian Naive Bayes.}
Our synthesizer releases a set of DP sufficient statistics and samples from a class-conditional model: it releases weighted label frequencies, categorical frequencies conditioned on the label, and (for continuous features) label-conditional first and second moments. Each per-record contribution is $\ell_2$ clipped, and Gaussian noise is added to each query block; the overall release is DP by composition.

\subsection{Attacks and metrics}
\paragraph{MIA and DOMIAS-style attack.}
We report membership inference advantage $2\cdot \mathrm{AUC}-1$ (absolute value). In addition to a distance-to-synthetic baseline, we implement a DOMIAS-style attack using a held-out reference set: for a query record, compute the ratio of average $k$NN distances to synthetic data versus to the reference set (we use $k=5$). We use the test split as the reference set of non-members.

\subsection{Per-decile risk curves and scorer diagnostics}
\begin{figure*}[t]
\centering
\includegraphics[width=0.48\linewidth]{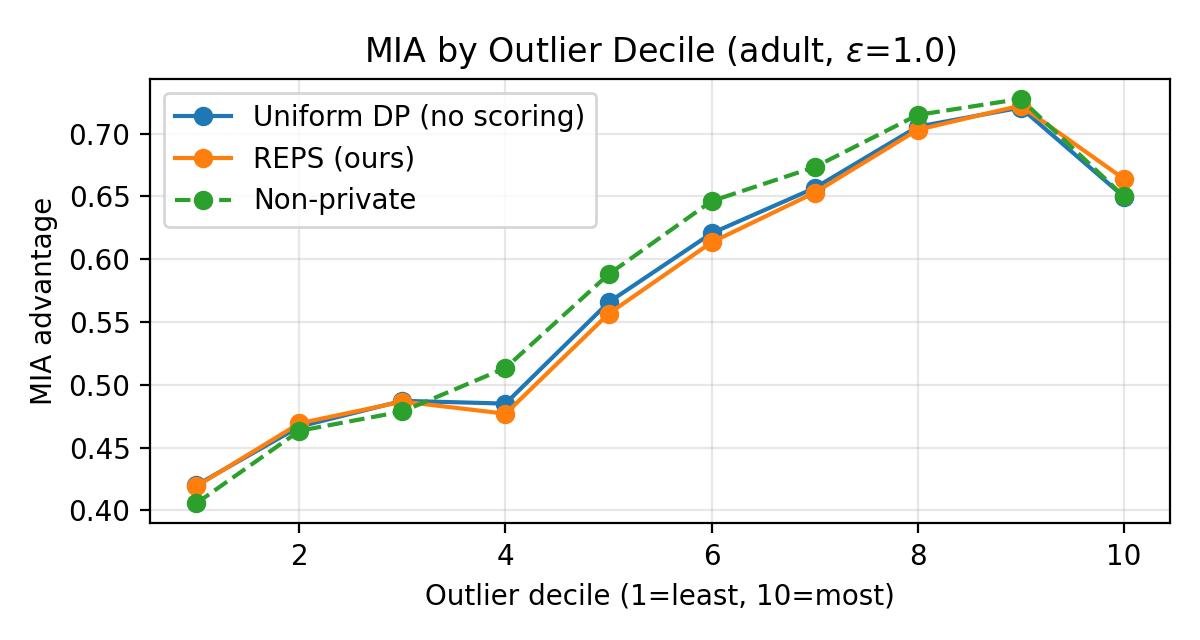}\hfill
\includegraphics[width=0.48\linewidth]{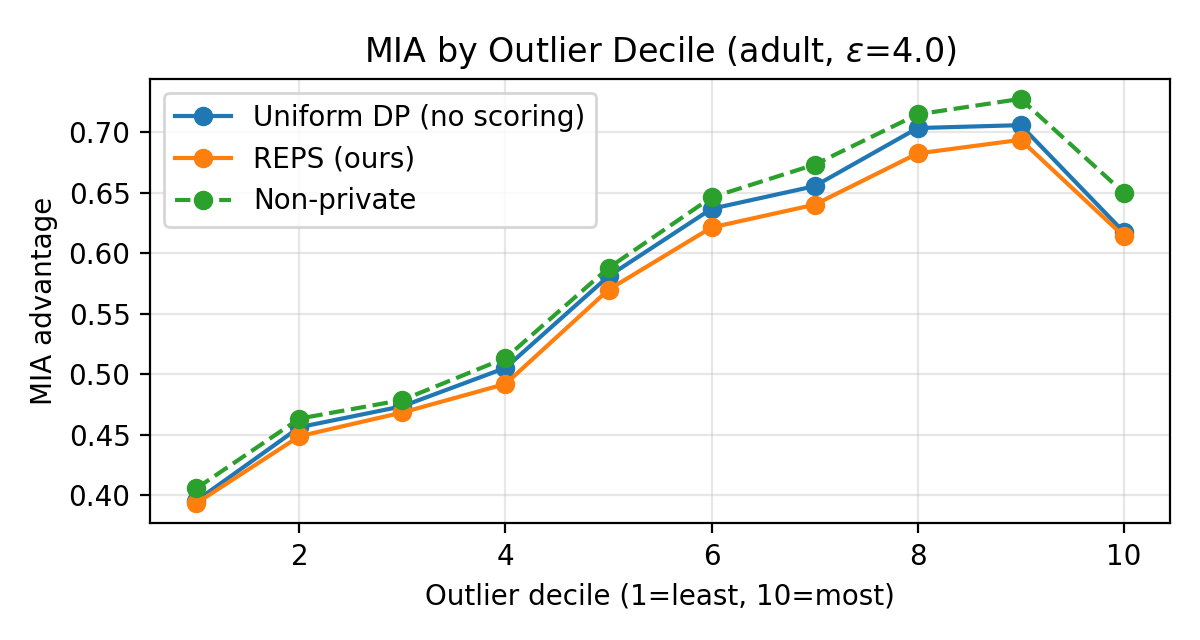}\\
\includegraphics[width=0.48\linewidth]{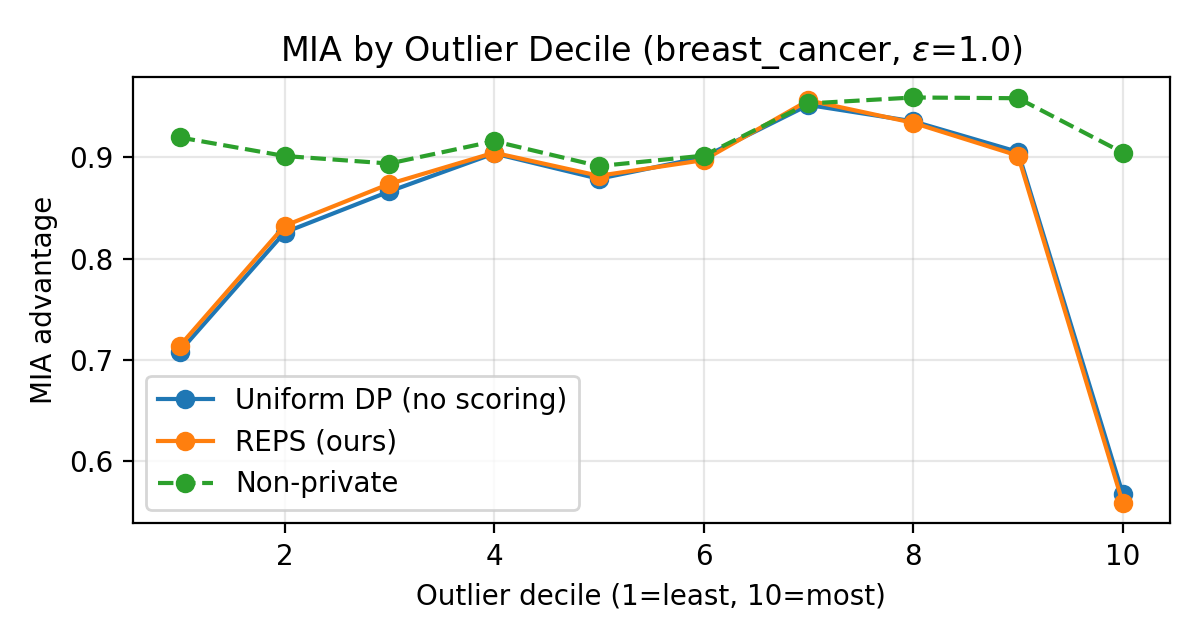}\hfill
\includegraphics[width=0.48\linewidth]{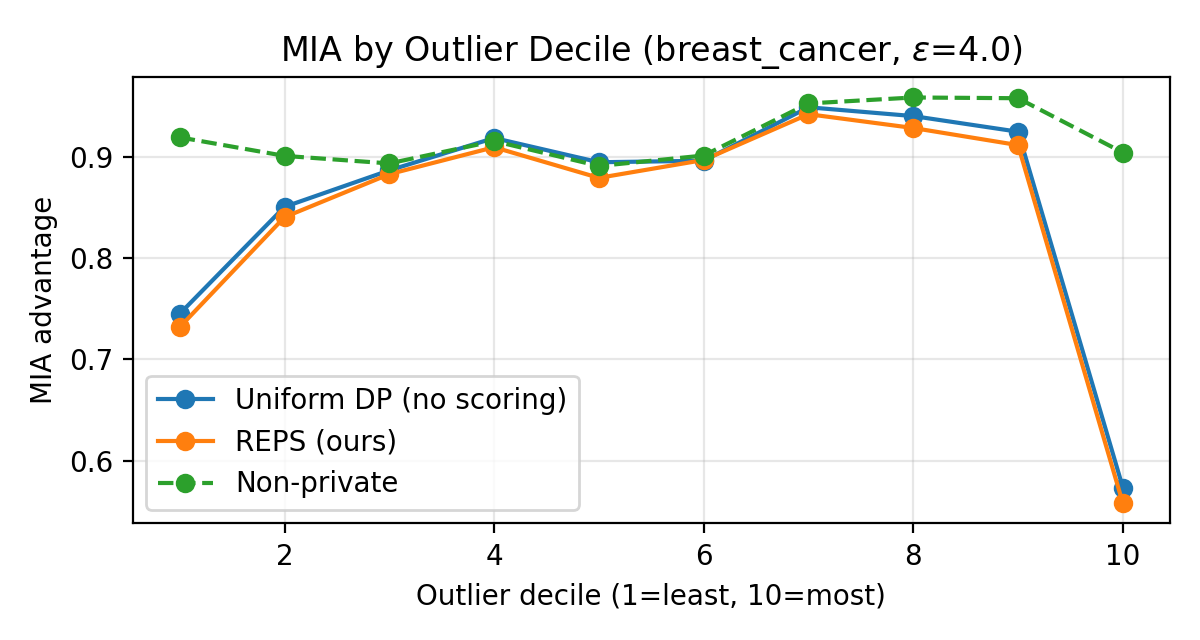}
\caption{MIA advantage by outlier decile (DOMIAS-style), for Adult and Breast Cancer at $\varepsilon\in\{1.0,4.0\}$. Lower curves indicate better privacy; the gap between decile 10 and decile 1 reflects privacy inequity.}
\label{fig:decile-curves}
\end{figure*}

\begin{figure*}[t]
\centering
\includegraphics[width=0.48\linewidth]{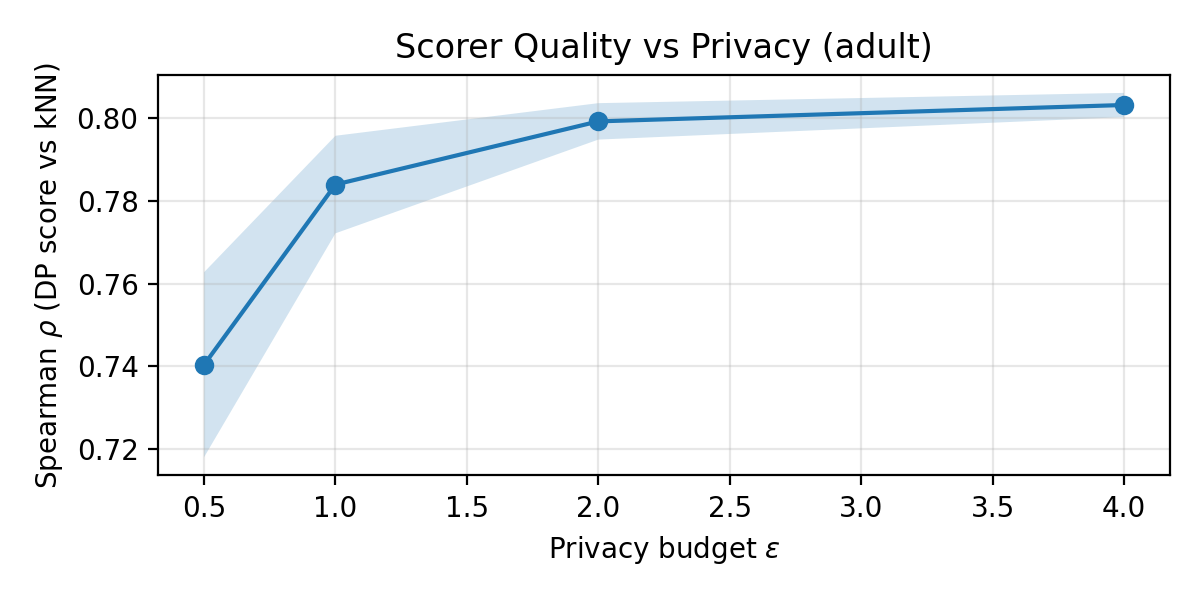}\hfill
\includegraphics[width=0.48\linewidth]{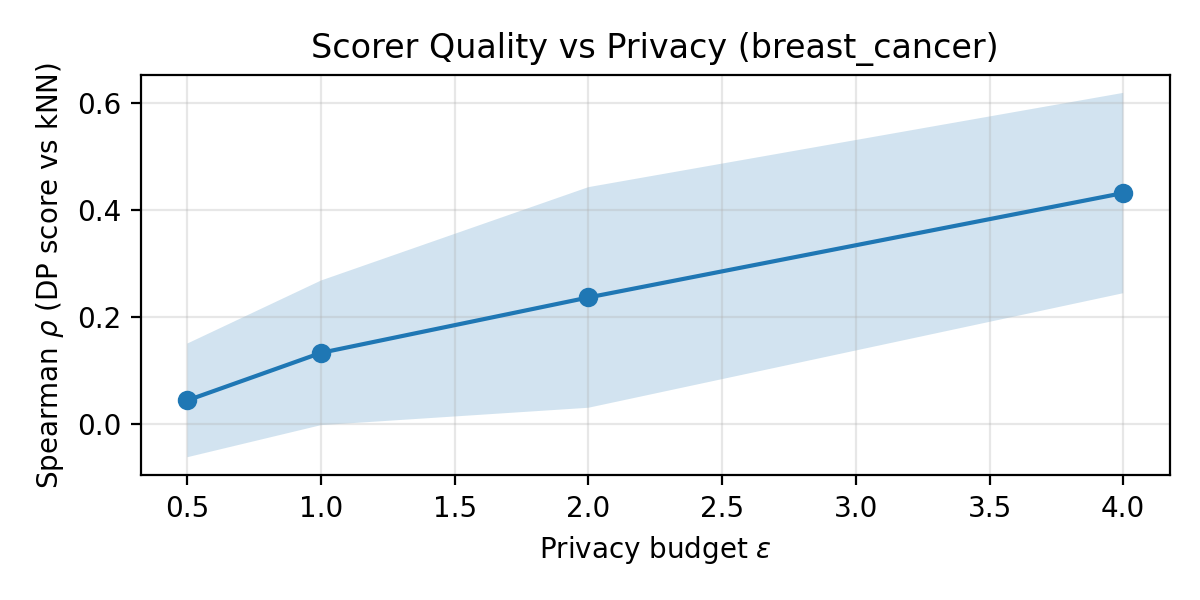}
\caption{Scorer quality versus $\varepsilon$: Spearman correlation between DP histogram scores and non-private $k$NN outlier scores. Higher correlation indicates the DP scorer successfully identifies true outliers.}
\label{fig:scorer-corr}
\end{figure*}

\begin{figure*}[t]
\centering
\includegraphics[width=0.48\linewidth]{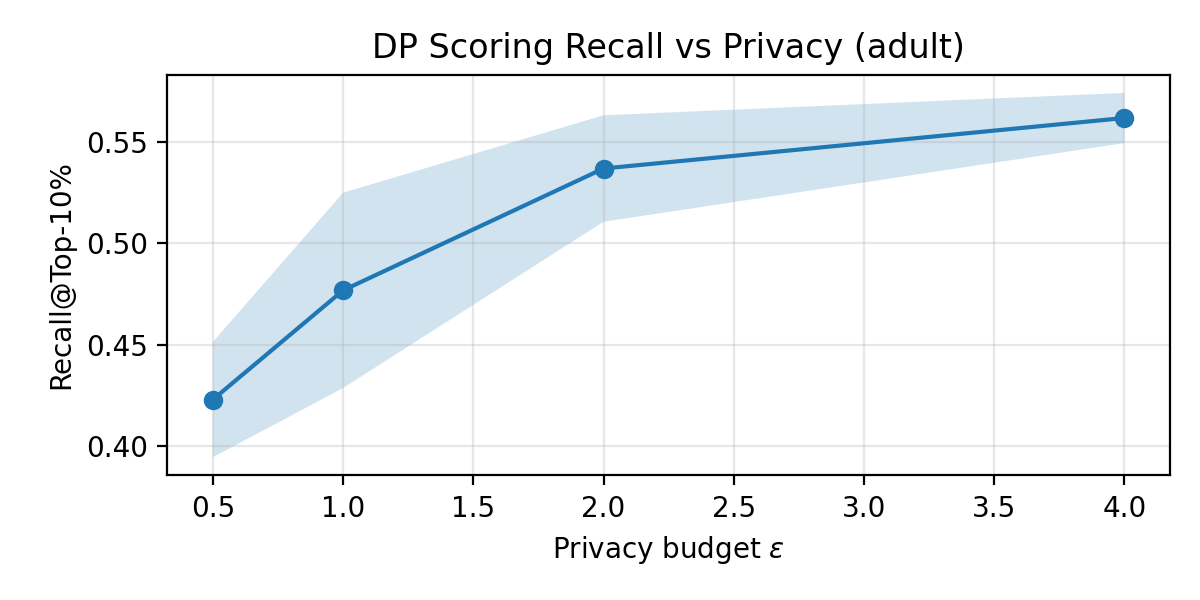}\hfill
\includegraphics[width=0.48\linewidth]{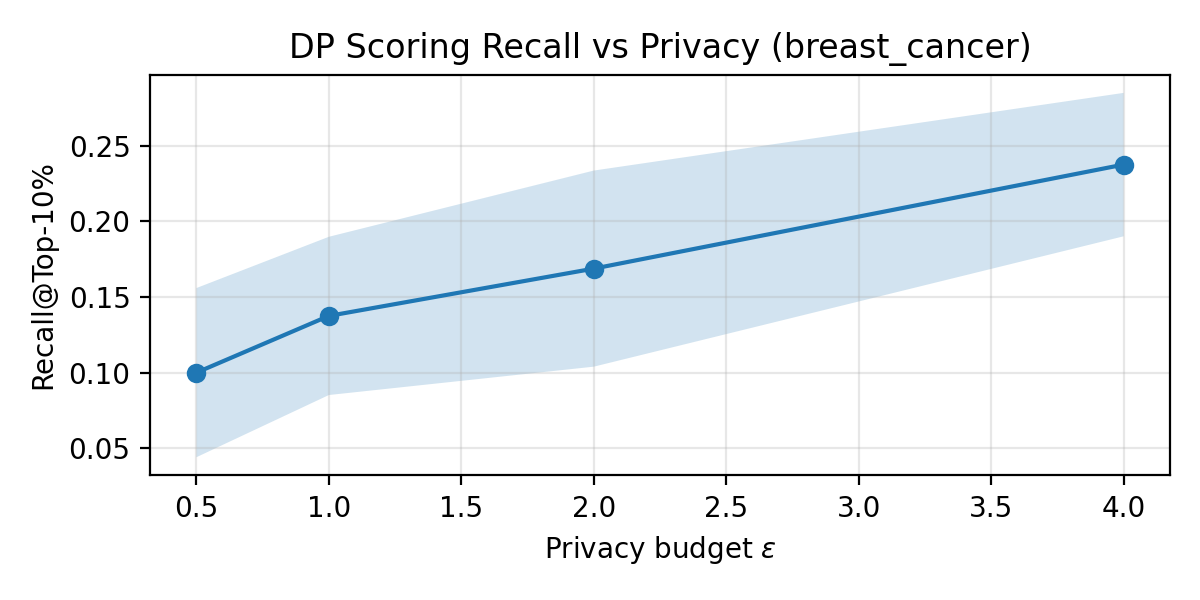}
\caption{Scorer quality versus $\varepsilon$: Recall@Top-10\% between the DP-score top decile and the $k$NN top decile. Higher recall means the DP scorer correctly flags the same high-risk records as the oracle.}
\label{fig:scorer-recall}
\end{figure*}

\begin{figure*}[t]
\centering
\includegraphics[width=0.48\linewidth]{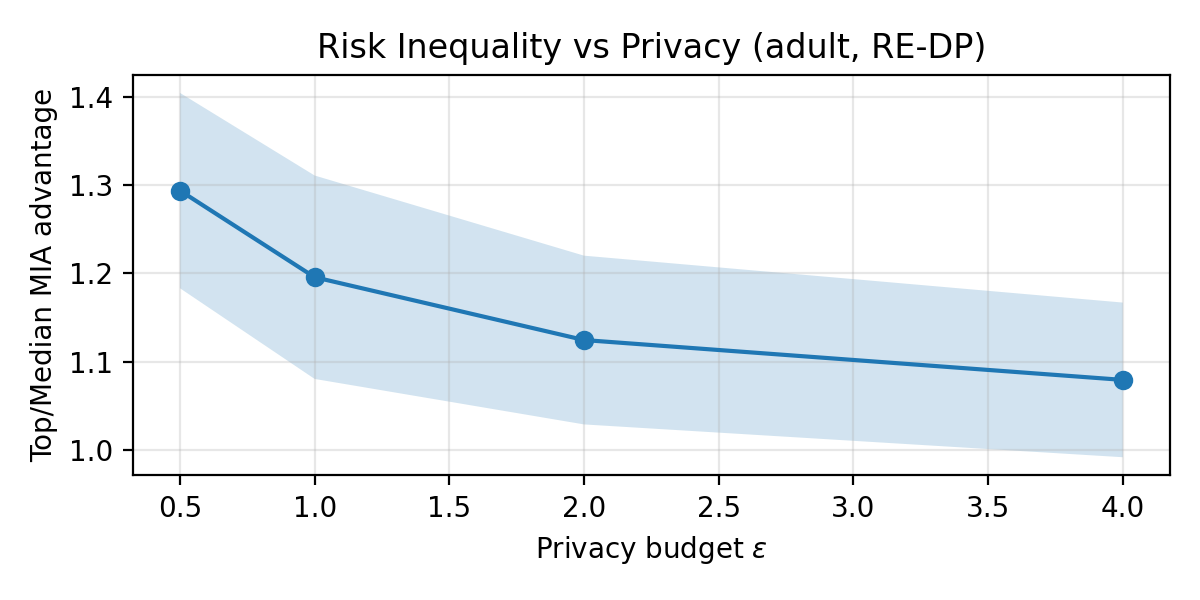}\hfill
\includegraphics[width=0.48\linewidth]{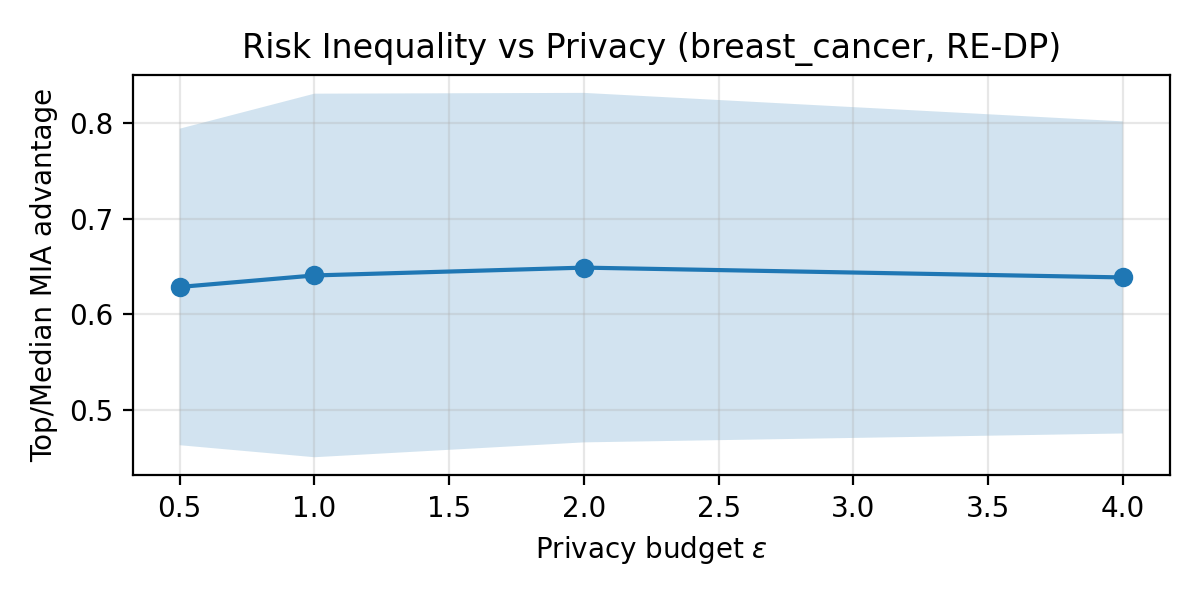}
\caption{Risk inequality ratio (top-decile MIA advantage divided by median-decile advantage) versus $\varepsilon$. Values closer to 1 indicate more equalized privacy protection across the population.}
\label{fig:ineq}
\end{figure*}

\begin{table*}[t]
\centering
\small
\caption{Results on sim dataset (mean over 5 seeds). Bold indicates best DP method per $\varepsilon$.}
\label{tab:sim:results}
\vspace{0.5em}
\begin{tabular}{@{}lcccccc@{}}
\toprule
$\varepsilon$ & Method & $\gamma$ & TSTR $\uparrow$ & MIA Adv $\downarrow$ & Top-10\% Adv $\downarrow$ & Top/Med $\downarrow$ \\
\midrule
\multirow{6}{*}{0.5}
 & Non-private & -- & 0.772 & 0.015 & 0.841 & 1.167 \\
 & DP Uniform (no scoring) & 0.0 & 0.604 & 0.006 & 0.667 & 1.025 \\
 & DP Uniform (split) & 0.0 & 0.589 & 0.007 & 0.666 & 1.015 \\
 & REPS (ours) & 0.5 & 0.569 & 0.008 & \textbf{0.659} & 1.009 \\
 & Random Downwt. & 0.5 & 0.581 & \textbf{0.005} & 0.663 & 1.005 \\
 & Hard Removal & 0.5 & \textbf{0.627} & \textbf{0.005} & 0.667 & \textbf{0.999} \\
\midrule
\multirow{6}{*}{1.0}
 & Non-private & -- & 0.772 & 0.015 & 0.841 & 1.167 \\
 & DP Uniform (no scoring) & 0.0 & \textbf{0.664} & 0.016 & 0.708 & 1.097 \\
 & DP Uniform (split) & 0.0 & 0.656 & 0.015 & 0.703 & 1.084 \\
 & REPS (ours) & 4.0 & 0.647 & \textbf{0.008} & \textbf{0.690} & 1.050 \\
 & Random Downwt. & 4.0 & 0.653 & 0.012 & 0.705 & 1.090 \\
 & Hard Removal & 4.0 & 0.627 & 0.013 & 0.695 & \textbf{1.024} \\
\midrule
\multirow{6}{*}{2.0}
 & Non-private & -- & 0.772 & 0.015 & 0.841 & 1.167 \\
 & DP Uniform (no scoring) & 0.0 & \textbf{0.704} & 0.008 & 0.731 & 1.088 \\
 & DP Uniform (split) & 0.0 & 0.697 & 0.006 & 0.737 & 1.091 \\
 & REPS (ours) & 4.0 & 0.689 & \textbf{0.005} & 0.708 & 1.086 \\
 & Random Downwt. & 4.0 & 0.691 & 0.010 & 0.718 & 1.090 \\
 & Hard Removal & 4.0 & 0.660 & 0.015 & \textbf{0.702} & \textbf{1.064} \\
\midrule
\multirow{6}{*}{4.0}
 & Non-private & -- & 0.772 & 0.015 & 0.841 & 1.167 \\
 & DP Uniform (no scoring) & 0.0 & \textbf{0.734} & \textbf{0.010} & 0.760 & 1.108 \\
 & DP Uniform (split) & 0.0 & 0.729 & 0.015 & 0.758 & 1.109 \\
 & REPS (ours) & 4.0 & 0.720 & 0.014 & 0.728 & 1.097 \\
 & Random Downwt. & 4.0 & 0.724 & 0.012 & 0.752 & 1.128 \\
 & Hard Removal & 4.0 & 0.690 & 0.013 & \textbf{0.720} & \textbf{1.089} \\
\bottomrule
\end{tabular}
\end{table*}

\begin{table*}[t]
\centering
\small
\caption{Results on breast\_cancer dataset (mean over 5 seeds). Bold indicates best DP method per $\varepsilon$.}
\label{tab:breast_cancer:results}
\vspace{0.5em}
\begin{tabular}{@{}lcccccc@{}}
\toprule
$\varepsilon$ & Method & $\gamma$ & TSTR $\uparrow$ & MIA Adv $\downarrow$ & Top-10\% Adv $\downarrow$ & Top/Med $\downarrow$ \\
\midrule
\multirow{6}{*}{0.5}
 & Non-private & -- & 0.989 & 0.038 & 0.904 & 1.016 \\
 & DP Uniform (no scoring) & 0.0 & \textbf{0.751} & \textbf{0.032} & 0.572 & 0.657 \\
 & DP Uniform (split) & 0.0 & 0.721 & 0.046 & 0.548 & \textbf{0.624} \\
 & REPS (ours) & 0.5 & 0.705 & 0.063 & 0.550 & 0.629 \\
 & Random Downwt. & 0.5 & 0.708 & 0.055 & \textbf{0.546} & 0.630 \\
 & Hard Removal & 0.5 & 0.707 & 0.054 & 0.550 & 0.628 \\
\midrule
\multirow{6}{*}{1.0}
 & Non-private & -- & 0.989 & 0.038 & 0.904 & 1.016 \\
 & DP Uniform (no scoring) & 0.0 & \textbf{0.914} & 0.049 & 0.568 & 0.653 \\
 & DP Uniform (split) & 0.0 & 0.910 & 0.047 & 0.558 & \textbf{0.640} \\
 & REPS (ours) & 4.0 & 0.821 & 0.040 & 0.559 & 0.641 \\
 & Random Downwt. & 4.0 & 0.817 & 0.047 & 0.564 & 0.652 \\
 & Hard Removal & 4.0 & 0.781 & \textbf{0.022} & \textbf{0.554} & 0.647 \\
\midrule
\multirow{6}{*}{2.0}
 & Non-private & -- & 0.989 & 0.038 & 0.904 & 1.016 \\
 & DP Uniform (no scoring) & 0.0 & \textbf{0.959} & 0.042 & 0.585 & 0.661 \\
 & DP Uniform (split) & 0.0 & 0.958 & 0.027 & 0.573 & \textbf{0.643} \\
 & REPS (ours) & 4.0 & 0.946 & 0.034 & 0.568 & 0.649 \\
 & Random Downwt. & 4.0 & 0.954 & 0.034 & \textbf{0.562} & \textbf{0.643} \\
 & Hard Removal & 4.0 & 0.946 & \textbf{0.026} & 0.574 & 0.655 \\
\midrule
\multirow{6}{*}{4.0}
 & Non-private & -- & 0.989 & 0.038 & 0.904 & 1.016 \\
 & DP Uniform (no scoring) & 0.0 & \textbf{0.966} & \textbf{0.023} & 0.574 & 0.644 \\
 & DP Uniform (split) & 0.0 & 0.965 & 0.037 & 0.585 & 0.677 \\
 & REPS (ours) & 4.0 & 0.963 & 0.034 & \textbf{0.558} & \textbf{0.639} \\
 & Random Downwt. & 4.0 & 0.961 & 0.035 & 0.565 & 0.645 \\
 & Hard Removal & 4.0 & 0.965 & 0.045 & 0.565 & 0.645 \\
\bottomrule
\end{tabular}
\end{table*}

\begin{table*}[t]
\centering
\small
\caption{Results on adult dataset (mean over 5 seeds). Bold indicates best DP method per $\varepsilon$.}
\label{tab:adult:results}
\vspace{0.5em}
\begin{tabular}{@{}lcccccc@{}}
\toprule
$\varepsilon$ & Method & $\gamma$ & TSTR $\uparrow$ & MIA Adv $\downarrow$ & Top-10\% Adv $\downarrow$ & Top/Med $\downarrow$ \\
\midrule
\multirow{6}{*}{0.5}
 & Non-private & -- & 0.890 & 0.008 & 0.650 & 1.107 \\
 & DP Uniform (no scoring) & 0.0 & 0.853 & \textbf{0.007} & \textbf{0.683} & \textbf{1.246} \\
 & DP Uniform (split) & 0.0 & 0.851 & \textbf{0.007} & 0.689 & 1.266 \\
 & REPS (ours) & 4.0 & 0.845 & 0.008 & 0.695 & 1.294 \\
 & Random Downwt. & 4.0 & 0.850 & \textbf{0.007} & 0.695 & 1.285 \\
 & Hard Removal & 4.0 & \textbf{0.855} & \textbf{0.007} & 0.722 & 1.351 \\
\midrule
\multirow{6}{*}{1.0}
 & Non-private & -- & 0.890 & 0.008 & 0.650 & 1.107 \\
 & DP Uniform (no scoring) & 0.0 & 0.863 & \textbf{0.007} & \textbf{0.650} & \textbf{1.151} \\
 & DP Uniform (split) & 0.0 & 0.860 & 0.008 & 0.656 & 1.166 \\
 & REPS (ours) & 4.0 & 0.854 & 0.008 & 0.664 & 1.196 \\
 & Random Downwt. & 4.0 & 0.859 & \textbf{0.007} & 0.662 & 1.180 \\
 & Hard Removal & 4.0 & \textbf{0.864} & 0.009 & 0.687 & 1.245 \\
\midrule
\multirow{6}{*}{2.0}
 & Non-private & -- & 0.890 & 0.008 & 0.650 & 1.107 \\
 & DP Uniform (no scoring) & 0.0 & \textbf{0.878} & \textbf{0.005} & \textbf{0.627} & \textbf{1.087} \\
 & DP Uniform (split) & 0.0 & 0.876 & 0.006 & 0.632 & 1.098 \\
 & REPS (ours) & 4.0 & 0.870 & 0.008 & 0.635 & 1.124 \\
 & Random Downwt. & 4.0 & 0.874 & 0.006 & 0.636 & 1.110 \\
 & Hard Removal & 4.0 & 0.870 & 0.008 & 0.669 & 1.171 \\
\midrule
\multirow{6}{*}{4.0}
 & Non-private & -- & 0.890 & 0.008 & 0.650 & 1.107 \\
 & DP Uniform (no scoring) & 0.0 & \textbf{0.889} & \textbf{0.006} & 0.618 & 1.064 \\
 & DP Uniform (split) & 0.0 & 0.887 & 0.007 & \textbf{0.613} & \textbf{1.060} \\
 & REPS (ours) & 4.0 & 0.882 & 0.007 & 0.614 & 1.079 \\
 & Random Downwt. & 4.0 & 0.884 & 0.007 & 0.615 & 1.062 \\
 & Hard Removal & 4.0 & 0.880 & 0.008 & 0.632 & 1.110 \\
\bottomrule
\end{tabular}
\end{table*}

\begin{table*}[t]
\centering
\small
\caption{Results on credit dataset (mean over 5 seeds). Bold indicates best DP method per $\varepsilon$.}
\label{tab:credit:results}
\vspace{0.5em}
\begin{tabular}{@{}lcccccc@{}}
\toprule
$\varepsilon$ & Method & $\gamma$ & TSTR $\uparrow$ & MIA Adv $\downarrow$ & Top-10\% Adv $\downarrow$ & Top/Med $\downarrow$ \\
\midrule
\multirow{6}{*}{0.5}
 & Non-private & -- & 0.781 & 0.058 & 0.965 & 0.982 \\
 & DP Uniform (no scoring) & 0.0 & \textbf{0.544} & 0.051 & \textbf{0.924} & 0.968 \\
 & DP Uniform (split) & 0.0 & 0.541 & 0.054 & 0.927 & 0.970 \\
 & REPS (ours) & 4.0 & 0.537 & 0.054 & 0.927 & 0.969 \\
 & Random Downwt. & 4.0 & 0.539 & \textbf{0.041} & 0.925 & 0.966 \\
 & Hard Removal & 4.0 & 0.535 & 0.053 & 0.925 & \textbf{0.963} \\
\midrule
\multirow{6}{*}{1.0}
 & Non-private & -- & 0.781 & 0.058 & 0.965 & 0.982 \\
 & DP Uniform (no scoring) & 0.0 & 0.570 & \textbf{0.036} & 0.937 & 0.963 \\
 & DP Uniform (split) & 0.0 & \textbf{0.574} & \textbf{0.036} & 0.930 & 0.964 \\
 & REPS (ours) & 4.0 & 0.573 & 0.058 & 0.934 & 0.975 \\
 & Random Downwt. & 4.0 & 0.571 & 0.049 & 0.928 & 0.967 \\
 & Hard Removal & 4.0 & 0.570 & 0.039 & \textbf{0.925} & \textbf{0.957} \\
\midrule
\multirow{6}{*}{2.0}
 & Non-private & -- & 0.781 & 0.058 & 0.965 & 0.982 \\
 & DP Uniform (no scoring) & 0.0 & 0.643 & 0.042 & 0.950 & 0.974 \\
 & DP Uniform (split) & 0.0 & 0.636 & 0.042 & 0.948 & \textbf{0.970} \\
 & REPS (ours) & 0.5 & 0.641 & 0.033 & 0.949 & 0.975 \\
 & Random Downwt. & 0.5 & 0.629 & \textbf{0.024} & 0.946 & 0.977 \\
 & Hard Removal & 0.5 & \textbf{0.648} & 0.028 & \textbf{0.944} & 0.974 \\
\midrule
\multirow{6}{*}{4.0}
 & Non-private & -- & 0.781 & 0.058 & 0.965 & 0.982 \\
 & DP Uniform (no scoring) & 0.0 & \textbf{0.689} & 0.056 & 0.965 & 0.989 \\
 & DP Uniform (split) & 0.0 & 0.682 & \textbf{0.040} & 0.962 & 0.986 \\
 & REPS (ours) & 4.0 & 0.682 & \textbf{0.040} & 0.961 & 0.984 \\
 & Random Downwt. & 4.0 & 0.677 & 0.041 & \textbf{0.960} & \textbf{0.983} \\
 & Hard Removal & 4.0 & 0.682 & 0.044 & 0.961 & 0.986 \\
\bottomrule
\end{tabular}
\end{table*}

\end{document}